\documentclass[11pt]{article}
\usepackage[margin=1in]{geometry}
\usepackage{amsmath,amssymb,amsthm,mathtools,bm}
\usepackage{booktabs,multirow,array}
\usepackage{algorithm}
\usepackage{algpseudocode}
\usepackage{graphicx}
\usepackage{xcolor}
\usepackage{hyperref}
\usepackage{cleveref}
\usepackage{enumitem}
\usepackage{microtype}
\usepackage{natbib}
\usepackage{tikz}
\usetikzlibrary{arrows.meta,positioning,calc}

\hypersetup{colorlinks=true,linkcolor=blue,citecolor=blue,urlcolor=blue}

\newtheorem{theorem}{Theorem}
\newtheorem{proposition}[theorem]{Proposition}
\newtheorem{lemma}[theorem]{Lemma}
\newtheorem{corollary}[theorem]{Corollary}
\newtheorem{assumption}[theorem]{Assumption}
\newtheorem{definition}[theorem]{Definition}
\newtheorem{remark}[theorem]{Remark}

\newcommand{\Qcal}{\mathcal{Q}}
\newcommand{\Mcal}{\mathcal{M}}
\newcommand{\Fcal}{\mathcal{F}}
\newcommand{\Hcal}{\mathcal{H}}
\newcommand{\Dcal}{\mathcal{D}}
\newcommand{\E}{\mathbb{E}}

\newcommand{\doop}{\operatorname{do}}

\newcommand{\argmax}{\operatorname*{arg\,max}}
\newcommand{\argmin}{\operatorname*{arg\,min}}
\newcommand{\diam}{\operatorname{diam}}

\title{Adversarial Causal Intervention Falsification:\\Learning Structural Generators by Selecting the Experiments That Expose Them}

\author{
  Mojtaba Eslami \\[4pt]
  \small University of Calgary \\[4pt]  
  \small \texttt{mojtaba.eslami@alumni.ucalgary.ca}
}

\date{}

\begin{document}
\maketitle

\begin{abstract}
Generative models can reproduce an observational distribution while encoding an incorrect causal structure. We study a sequential game in which a structural causal generator proposes observational and interventional distributions, while an adversarial experimentalist selects interventions intended to maximally falsify the generator. The discriminator is therefore not merely a real-versus-synthetic classifier: it is indexed by an intervention and tests whether the generator reproduces the corresponding post-intervention law. We introduce Adversarial Causal Intervention Falsification (ACIF), formulate oracle and implementable versions of the game, and distinguish three objects that are often conflated: observational fit, interventional equivalence over an admissible query class, and point identification of a structural causal model. For finite model and intervention classes, we prove: (i) an exact reduction of the adversarial objective to a worst-intervention integral probability metric; (ii) identification up to interventional equivalence, with point identification under a separating intervention family; (iii) existence of mixed-strategy equilibria; (iv) finite-sample uniform convergence and margin-based model-selection guarantees; and (v) a logarithmic elimination guarantee for a disagreement-driven sequential design under a balanced-separation condition. We also give a complete linear-Gaussian example in which two observationally indistinguishable causal directions are separated by a single well-chosen intervention. The framework clarifies what an adversarial causal discriminator can and cannot certify, and provides a principled bridge between causal generative modeling, active causal discovery, and experimental design.
\end{abstract}

\noindent\textbf{Keywords:} causal discovery; structural causal models; generative adversarial networks; active learning; intervention design; model falsification; integral probability metrics.

\section{Introduction}

A flexible generator can match the observational distribution of a system without learning its causal organization. This is not a defect of neural networks; it is a consequence of causal non-identifiability. Distinct structural causal models (SCMs) can induce the same joint observational distribution while making incompatible predictions under intervention. Consequently, a discriminator trained only to distinguish observed samples from synthetic observational samples cannot certify causal validity.

This paper develops a different adversarial game. A \emph{structural generator} proposes a causal data-generating process. An \emph{experimental adversary} chooses an intervention. An intervention-indexed critic then tries to distinguish samples generated by the true intervened system from samples generated by the proposed SCM under the same intervention. The generator survives only if it matches the target system across the interventions selected by its strongest adversary.

The proposal is related to, but distinct from, three neighboring literatures. Causal generative models impose causal structure on deep generators and can produce interventional or counterfactual samples \citep{kocaoglu2018causalgan, pawlowski2020dscm, xia2021causalnf}. Active causal discovery selects informative experiments to orient edges or reduce uncertainty over causal graphs \citep{he2008active, hauser2012active, agrawal2019abcd, tigas2023differentiable}. Adversarial goodness-of-fit methods search for discriminators that expose discrepancies between real and simulated data \citep{goodfellow2014gan, arjovsky2017wasserstein, drouin2025act}. ACIF combines these ideas around a specific estimand: the worst post-intervention discrepancy over a declared class of feasible causal queries.

The conceptual contribution is to replace the informal statement ``the discriminator verifies causality'' with a precise statement: the discriminator can certify, at best, \emph{interventional equivalence over the query class it is allowed to test}. Point identification requires the query class to separate the candidate SCMs. This distinction determines both the theorems and the algorithm.

\paragraph{Contributions.}
\begin{enumerate}[leftmargin=1.5em]
    \item We define an intervention-indexed adversarial objective whose population value is the largest integral probability metric (IPM) discrepancy between the true and generated post-intervention laws.
    \item We characterize the zero set of the game as an interventional equivalence class and give conditions under which the true SCM is uniquely identified.
    \item We introduce a prospective, implementable selector that chooses interventions using disagreement among surviving generators, since the true post-intervention law is unavailable before the experiment is conducted.
    \item We establish uniform convergence, margin-based recovery, mixed-strategy equilibrium, and sequential elimination guarantees.
    \item We provide worked discrete and linear-Gaussian examples, practical algorithms, failure modes, and an empirical protocol suitable for evaluating the method without overstating causal identification.
\end{enumerate}

\section{Background and related work}

\subsection{Structural causal models and interventions}

An SCM $M$ over variables $X=(X_1,\ldots,X_d)$ consists of a directed acyclic graph $G$, exogenous variables $U=(U_1,\ldots,U_d)$, and assignments
\begin{equation}
X_j=f_j(X_{\operatorname{pa}_G(j)},U_j), \qquad j=1,\ldots,d.
\end{equation}
Under causal sufficiency the $U_j$ are jointly independent. A hard intervention $q=(S,a)$ replaces the structural assignments for $j\in S$ by $X_j=a_j$. The resulting distribution is denoted $P_M^q$. We include the null intervention $q=\varnothing$, for which $P_M^{\varnothing}$ is observational.

Observational Markov equivalence prevents identification of a unique DAG from conditional independences alone. Interventions refine observational equivalence into interventional Markov equivalence \citep{hauser2012characterization}. Active intervention design exploits this refinement by choosing targets that orient unresolved edges or maximize expected information gain \citep{hauser2012active, agrawal2019abcd, tigas2023differentiable}.

\subsection{Causal generative models}

CausalGAN showed that an adversarially trained generator, structured according to a supplied causal graph, can reproduce observational and interventional distributions under idealized conditions \citep{kocaoglu2018causalgan}. Deep SCMs, causal normalizing flows, and related models extend this program to high-dimensional and counterfactual settings \citep{pawlowski2020dscm, xia2021causalnf, khemakhem2021causal}. These methods answer causal queries only relative to their graph, structural restrictions, and data support. ACIF is complementary: it asks which intervention should be performed to most effectively challenge the current causal generator.

\subsection{Adversarial testing and active experimentation}

A GAN optimizes a discrepancy between a data distribution and a generated distribution through a learned critic \citep{goodfellow2014gan}. Wasserstein GANs and IPM-based generative models make the discrepancy interpretation explicit \citep{arjovsky2017wasserstein, sriperumbudur2012empirical}. Recent adversarial causal tuning searches jointly over causal simulation pipelines and discriminators that expose poor fit \citep{drouin2025act}. Our focus is narrower and more causal: the adversary acts in the space of interventions, and the theoretical target is a family of post-intervention laws.

\subsection{Differentiable and active causal discovery}

A second relevant thread treats intervention targets as part of the estimation problem rather than as fixed metadata. Differentiable causal discovery from interventional data (DCDI) formulates structure learning from mixed observational and interventional samples as a continuous, augmented-Lagrangian-constrained optimization over a weighted adjacency matrix, and shows that interventional data materially improves identifiability relative to purely observational scores \citep{brouillard2020dcdi}. Building on this line, active intervention targeting (AIT) learns which nodes to intervene on next, using a differentiable proxy for graph uncertainty, and shows empirically that adaptively chosen interventions reach the correct structure with substantially fewer experiments than randomly targeted ones \citep{scherrer2021aitcausal}. ACIF differs from both in what it treats as the object of falsification: DCDI and AIT search over graphs using a likelihood-style or adjacency-matrix score, whereas ACIF poses the selection problem directly in terms of an adversarially chosen, intervention-indexed two-sample discrepancy, and it characterizes the population game (equivalence classes, separating families, mixed-strategy equilibria) rather than only the estimator. Sections \ref{sec:sequential} and \ref{sec:numerical-illustration} make the connection to this literature concrete: the disagreement acquisition rule in \eqref{eq:disagreement} is the ACIF analogue of the acquisition functions used by AIT and by earlier score-based active-structure-learning methods \citep{he2008active, hauser2012active, agrawal2019abcd, tigas2023differentiable}, specialized to a worst-case, critic-detectable notion of disagreement rather than an expected information gain.

The finite-model sequential-elimination result in \Cref{thm:elimination} is also best understood against the classical theory of query-based active learning. Generalized binary search (GBS) shows that a greedy rule which repeatedly queries the point most evenly splitting the surviving hypothesis set identifies the truth in $O(\log|\mathcal H|)$ queries whenever a geometric ``neighborliness'' or balanced-split condition holds, and that this rate is information-theoretically optimal \citep{nowak2011gbs}. \Cref{ass:balanced} is the causal-falsification analogue of that condition: it asks that some affordable intervention split the surviving SCMs into two well-separated, comparably sized groups. \Cref{thm:elimination} is therefore not a new information-theoretic result so much as a transplant of the GBS argument into the setting where queries are interventions and disagreement is measured by an intervention-indexed IPM rather than by a binary label. This connection also explains \emph{when} the guarantee can fail: path-like or line-shaped hypothesis classes, in which no single intervention removes more than an $O(1/|\mathcal M|)$ fraction of candidates, violate balanced separation and force near-linear elimination, exactly as in the classical GBS lower bounds. \Cref{sec:numerical-illustration} exhibits both regimes on the same worked family of models.

\section{Problem formulation}

\subsection{Candidate generators, interventions, and critics}

Let $M_\star\in\Mcal$ be the unknown true SCM. Let $\Mcal$ be a candidate class of structural generators and $\Qcal$ a set of admissible interventions. An intervention may encode a target set, value, duration, environment, cost, or safety restriction. Let $c(q)\geq 0$ denote its cost.

For each $q\in\Qcal$, let $\Fcal_q$ be a symmetric class of measurable critic functions $f:\mathcal X\to[-B,B]$. Define the intervention-specific IPM
\begin{equation}
 d_q(M_\star,M)
 :=\sup_{f\in\Fcal_q}
 \left|\E_{X\sim P_{M_\star}^q}f(X)-\E_{X\sim P_M^q}f(X)\right|.
 \label{eq:ipm}
\end{equation}
Examples include total variation, maximum mean discrepancy, and Wasserstein-1 distance under suitable choices of $\Fcal_q$.

\begin{definition}[ACIF population value]
For a penalty parameter $\lambda\geq 0$, define
\begin{equation}
V_\lambda(M)
:=\sup_{q\in\Qcal}\left\{d_q(M_\star,M)-\lambda c(q)\right\}.
\label{eq:population-value}
\end{equation}
The ACIF estimator minimizes $V_\lambda(M)$ over $M\in\Mcal$.
\end{definition}

When $\lambda=0$, the adversary selects the intervention with the largest post-intervention discrepancy. For $\lambda>0$, the adversary balances falsification power against cost.

\subsection{Oracle and prospective games}

The objective in \eqref{eq:population-value} is an oracle objective: it assumes access to $P_{M_\star}^q$ for every $q$. In a real experiment, that distribution is unknown until intervention $q$ is conducted. We therefore distinguish two settings.

\paragraph{Retrospective ACIF.} A dataset already contains samples from several environments or interventions. The adversary reweights or selects among these observed environments to expose generator misspecification.

\paragraph{Prospective ACIF.} Before collecting data under a new intervention, the selector uses an ensemble or version space $\mathcal V_t\subseteq\Mcal$ and chooses the intervention on which surviving models disagree most:
\begin{equation}
A_t(q)
:=\sup_{M,M'\in\mathcal V_t} d_q(M,M')-\lambda c(q),
\qquad
q_t\in\argmax_{q\in\Qcal}A_t(q).
\label{eq:disagreement}
\end{equation}
After observing data from $P_{M_\star}^{q_t}$, models inconsistent with the new evidence are downweighted or removed.

This distinction is essential. A selector cannot maximize the unknown true discrepancy before experimentation; it can maximize predicted disagreement, expected information gain, or a robust lower bound on expected falsification power.

\subsection{Interventional equivalence}

\begin{definition}[$\Qcal$-interventional equivalence]
Two SCMs $M$ and $M'$ are equivalent over $\Qcal$, written $M\equiv_{\Qcal,\Fcal}M'$, if
\begin{equation}
 d_q(M,M')=0 \quad \text{for every }q\in\Qcal.
\end{equation}
When every $\Fcal_q$ is measure determining, this is equivalent to $P_M^q=P_{M'}^q$ for all $q\in\Qcal$.
\end{definition}

\begin{definition}[Separating intervention family]
The family $\Qcal$ is separating for $\Mcal$ relative to $\Fcal$ if for every distinct $M,M'\in\Mcal$, there exists $q\in\Qcal$ such that $d_q(M,M')>0$.
\end{definition}

The definition makes clear that causal identification is a joint property of the model class, available interventions, and critic richness.

\section{Population theory}

\subsection{The adversarial reduction}

Define the signed critic payoff
\begin{equation}
L(M,q,f)
=\E_{P_{M_\star}^q}f(X)-\E_{P_M^q}f(X)-\lambda c(q).
\end{equation}
Because $\Fcal_q$ is symmetric, absolute values can be absorbed by replacing $f$ with $-f$.

\begin{theorem}[Worst-intervention IPM representation]
\label{thm:ipm-reduction}
Suppose every $\Fcal_q$ is symmetric. Then
\begin{equation}
\sup_{q\in\Qcal}\sup_{f\in\Fcal_q}L(M,q,f)
=V_\lambda(M).
\end{equation}
Consequently, the oracle adversary chooses an intervention attaining the largest critic-detectable post-intervention discrepancy net of cost.
\end{theorem}

\begin{proof}
Fix $q$. By symmetry of $\Fcal_q$,
\[
\sup_{f\in\Fcal_q}\left(\E_{P_{M_\star}^q}f-\E_{P_M^q}f\right)
=\sup_{f\in\Fcal_q}\left|\E_{P_{M_\star}^q}f-\E_{P_M^q}f\right|
=d_q(M_\star,M).
\]
Subtracting the constant $\lambda c(q)$ and taking the supremum over $q$ gives \eqref{eq:population-value}.
\end{proof}

\subsection{Identification and its limits}

\begin{theorem}[Identification up to interventional equivalence]
\label{thm:equivalence}
Assume $M_\star\in\Mcal$, $\lambda=0$, and every $d_q$ is a pseudometric. Then
\begin{equation}
\argmin_{M\in\Mcal}V_0(M)
=
\{M\in\Mcal:M\equiv_{\Qcal,\Fcal}M_\star\}.
\end{equation}
The minimum value is zero.
\end{theorem}

\begin{proof}
Because each $d_q$ is nonnegative, $V_0(M)\geq 0$. Since $d_q(M_\star,M_\star)=0$ for all $q$, $V_0(M_\star)=0$, so the minimum is zero. A model $M$ attains zero if and only if $\sup_q d_q(M_\star,M)=0$. Nonnegativity implies this holds if and only if $d_q(M_\star,M)=0$ for every $q$, which is precisely $M\equiv_{\Qcal,\Fcal}M_\star$.
\end{proof}

\begin{corollary}[Point identification]
\label{cor:point-id}
Under the conditions of \Cref{thm:equivalence}, if $\Qcal$ is separating for $\Mcal$, then $M_\star$ is the unique minimizer of $V_0$.
\end{corollary}

\begin{proof}
If $M\neq M_\star$, separation gives a $q$ with $d_q(M,M_\star)>0$, hence $V_0(M)>0=V_0(M_\star)$.
\end{proof}

\begin{remark}[Why observational realism is insufficient]
If $\Qcal=\{\varnothing\}$ contains only the observational regime, \Cref{thm:equivalence} identifies only models that induce the same observational distribution. Even an infinitely powerful discriminator cannot distinguish observationally equivalent causal models.
\end{remark}

\begin{proposition}[Effect of intervention costs]
\label{prop:cost}
Let $\lambda>0$ and suppose the null intervention $q_0$ has $c(q_0)=0$. Then $V_\lambda(M_\star)=0$. A false model $M$ is distinguishable in the penalized game only if there exists $q$ such that
\begin{equation}
d_q(M_\star,M)>\lambda c(q).
\end{equation}
Thus cost penalization can intentionally enlarge the set of practically indistinguishable models.
\end{proposition}

\begin{proof}
For the true model every discrepancy is zero, so the supremum of $-\lambda c(q)$ is zero because $q_0$ is available. For a false model, $V_\lambda(M)>0$ exactly when some penalized discrepancy is positive.
\end{proof}

\subsection{Mixed strategies}

For finite $\Mcal$ and $\Qcal$, define $D(M,q)=d_q(M_\star,M)-\lambda c(q)$. A randomized generator uses a distribution $\mu\in\Delta(\Mcal)$ and a randomized adversary uses $\pi\in\Delta(\Qcal)$, with bilinear payoff
\begin{equation}
\Phi(\mu,\pi)=\sum_{M,q}\mu(M)\pi(q)D(M,q).
\end{equation}

\begin{theorem}[Finite mixed-strategy equilibrium]
\label{thm:minimax}
If $\Mcal$ and $\Qcal$ are finite, then
\begin{equation}
\min_{\mu\in\Delta(\Mcal)}\max_{\pi\in\Delta(\Qcal)}\Phi(\mu,\pi)
=
\max_{\pi\in\Delta(\Qcal)}\min_{\mu\in\Delta(\Mcal)}\Phi(\mu,\pi),
\end{equation}
and a saddle-point pair $(\mu^\star,\pi^\star)$ exists.
\end{theorem}

\begin{proof}
The simplices are compact and convex, and $\Phi$ is continuous and bilinear. The result follows from the finite-dimensional minimax theorem.
\end{proof}

A mixed intervention strategy is useful when no single experiment simultaneously separates all plausible models. The equilibrium distribution concentrates experimental budget on interventions that protect against the most difficult remaining alternatives.

\section{Finite-sample theory}

\subsection{Empirical objective}

Suppose that for each $q\in\Qcal$ we observe $n_q$ independent samples $X_{q,1},\ldots,X_{q,n_q}\sim P_{M_\star}^q$, and can generate $m_q$ independent samples $\widetilde X_{q,1}^{M},\ldots,\widetilde X_{q,m_q}^{M}\sim P_M^q$. Define
\begin{equation}
\widehat d_q(M_\star,M)
=
\sup_{f\in\Fcal_q}
\left|
\frac{1}{n_q}\sum_{i=1}^{n_q}f(X_{q,i})
-
\frac{1}{m_q}\sum_{i=1}^{m_q}f(\widetilde X_{q,i}^{M})
\right|,
\end{equation}
and
\begin{equation}
\widehat V_\lambda(M)
=
\max_{q\in\Qcal}\{\widehat d_q(M_\star,M)-\lambda c(q)\}.
\end{equation}

For a general critic class, define the expected Rademacher complexities
\begin{align}
\mathfrak R_{n_q}(\Fcal_q;P_{M_\star}^q)
&=\E\sup_{f\in\Fcal_q}\frac{1}{n_q}\sum_{i=1}^{n_q}\sigma_i f(X_{q,i}),\\
\mathfrak R_{m_q}(\Fcal_q;P_M^q)
&=\E\sup_{f\in\Fcal_q}\frac{1}{m_q}\sum_{i=1}^{m_q}\sigma_i f(\widetilde X_{q,i}^{M}).
\end{align}

\begin{theorem}[Uniform convergence over finite model and intervention classes]
\label{thm:uniform}
Assume $\Mcal$ and $\Qcal$ are finite and $|f(x)|\leq B$ for all $q$, $f\in\Fcal_q$, and $x$. With probability at least $1-\delta$, simultaneously for every $M\in\Mcal$,
\begin{align}
|\widehat V_\lambda(M)-V_\lambda(M)|
\leq
\max_{q\in\Qcal}
\Bigg[&2\mathfrak R_{n_q}(\Fcal_q;P_{M_\star}^q)
+2\mathfrak R_{m_q}(\Fcal_q;P_M^q)\\
&+B\sqrt{\frac{2\log(4|\Mcal||\Qcal|/\delta)}{n_q}}
+B\sqrt{\frac{2\log(4|\Mcal||\Qcal|/\delta)}{m_q}}
\Bigg].
\end{align}
\end{theorem}

\begin{proof}
For each fixed $(M,q)$, apply the standard symmetrization and bounded-difference bound separately to the real and generated empirical processes. A union bound over $|\Mcal||\Qcal|$ pairs gives simultaneous control of $|\widehat d_q-d_q|$. Finally,
\[
\left|\max_q a_q-\max_q b_q\right|\leq\max_q|a_q-b_q|,
\]
which transfers the bound to $\widehat V_\lambda$.
\end{proof}

\Cref{thm:uniform} is stated at the level of abstract Rademacher complexities so that it applies uniformly across critic classes; the rate it delivers, however, depends heavily on which class $\Fcal_q$ is chosen, and this choice is where the ``critic-detectable'' qualifier in \Cref{thm:ipm-reduction} does real work.

\begin{remark}[Rates for two standard critic classes]
\label{rem:rates}
Two cases recur in practice and make \Cref{thm:uniform} concrete.
\begin{enumerate}[leftmargin=1.5em]
\item \emph{RKHS critics (MMD).} If $\Fcal_q$ is the unit ball of a reproducing-kernel Hilbert space $\Hcal_q$ with a bounded kernel, $k(x,x)\leq \kappa^2$, then $d_q$ is the maximum mean discrepancy and $\mathfrak R_{n_q}(\Fcal_q;P)\leq \kappa/\sqrt{n_q}$ regardless of the ambient dimension of $X$ \citep{sriperumbudur2012empirical}. Substituting into \Cref{thm:uniform} gives a dimension-free rate of order $\kappa/\sqrt{n_q}+\kappa/\sqrt{m_q}$ plus the usual $\sqrt{\log(|\Mcal||\Qcal|/\delta)}$ union-bound term. This is the setting in which \Cref{cor:margin} is cheapest to satisfy: the sample complexity needed to beat a fixed margin $\Delta$ does not grow with the dimension of the post-intervention variables, only with $|\Mcal|$, $|\Qcal|$, and $\kappa/\Delta$.
\item \emph{Lipschitz critics (Wasserstein-1).} If $\Fcal_q$ is the class of $1$-Lipschitz functions on a domain of diameter $\diam(\mathcal X)$, then $d_q$ is the Wasserstein-1 distance, and $\mathfrak R_{n_q}(\Fcal_q;P)$ scales with the covering number of $\mathcal X$ and typically degrades as $n_q^{-1/D}$ in ambient dimension $D$ once $D\geq 3$ \citep{sriperumbudur2012empirical}. The IPM in this case detects distributional shape (not only the mean), which is why the binary-treatment example in \Cref{sec:worked-examples} uses it to expose disagreements invisible to an average-treatment-effect comparison, but the same expressiveness makes the finite-sample guarantee of \Cref{thm:uniform} substantially more sample-hungry in high dimension.
\end{enumerate}
In both cases the qualitative message of \Cref{thm:ipm-reduction} is unchanged: the adversary can only detect what its declared critic class can detect, and \Cref{rem:rates} quantifies the price of choosing an expressive class.
\end{remark}

\begin{corollary}[Margin-based exact recovery]
\label{cor:margin}
Suppose $M_\star$ is the unique population minimizer and
\begin{equation}
\Delta
:=\min_{M\neq M_\star}
\{V_\lambda(M)-V_\lambda(M_\star)\}>0.
\end{equation}
If the right-hand side of \Cref{thm:uniform} is at most $\Delta/2$ for every $M$, then every empirical minimizer of $\widehat V_\lambda$ equals $M_\star$.
\end{corollary}

\begin{proof}
Uniform error at most $\Delta/2$ implies, for every $M\neq M_\star$,
\[
\widehat V_\lambda(M)
\geq V_\lambda(M)-\Delta/2
\geq V_\lambda(M_\star)+\Delta/2
\geq \widehat V_\lambda(M_\star).
\]
Strict uniqueness follows when the uniform error is strictly below $\Delta/2$; with the weak inequality, any tie-breaking rule favoring the smaller confidence set recovers $M_\star$.
\end{proof}

\subsection{A simple bounded-critic sample complexity}

For transparent rates, suppose each $\Fcal_q$ is a finite class with $|\Fcal_q|\leq K$, generated samples are arbitrarily abundant, and $n_q=n$. Hoeffding's inequality and a union bound give the following.

\begin{corollary}[Finite critic class]
\label{cor:finite-class}
With probability at least $1-\delta$,
\begin{equation}
\sup_{M,q}|\widehat d_q(M_\star,M)-d_q(M_\star,M)|
\leq
2B\sqrt{\frac{2\log(2|\Mcal||\Qcal|K/\delta)}{n}}.
\end{equation}
Therefore, exact recovery under margin $\Delta$ is guaranteed when
\begin{equation}
n
\geq
\frac{32B^2}{\Delta^2}
\log\frac{2|\Mcal||\Qcal|K}{\delta}.
\end{equation}
\end{corollary}

\section{Sequential adversarial intervention selection}
\label{sec:sequential}

\subsection{Version-space algorithm}

Let $\widehat d_{q,t}(M_\star,M)$ be an empirical discrepancy after collecting data through round $t$. Let $\beta_t(q,M)$ be a valid confidence radius. Define the version space
\begin{equation}
\mathcal V_t
=
\left\{M\in\Mcal:
\widehat d_{q_s,t}(M_\star,M)\leq \beta_t(q_s,M)
\text{ for all }s\leq t
\right\}.
\end{equation}
The prospective adversary selects $q_{t+1}$ by pairwise disagreement as in \eqref{eq:disagreement}.

\begin{algorithm}[t]
\caption{Adversarial Causal Intervention Falsification (finite version space)}
\label{alg:acif}
\begin{algorithmic}[1]
\Require Candidate SCMs $\Mcal$, interventions $\Qcal$, cost $c$, penalty $\lambda$, confidence level $\delta$
\State Initialize $\mathcal V_0\gets\Mcal$
\For{$t=0,1,2,\ldots$}
    \If{$|\mathcal V_t|=1$}
        \State \Return the remaining model
    \EndIf
    \State $q_{t+1}\gets\argmax_{q\in\Qcal}\left\{\sup_{M,M'\in\mathcal V_t}\widehat d_q(M,M')-\lambda c(q)\right\}$
    \State Conduct intervention $q_{t+1}$ and collect a batch from the real system
    \State Train or update the intervention-indexed critic for $q_{t+1}$
    \State Remove models whose discrepancy exceeds the confidence threshold
\EndFor
\end{algorithmic}
\end{algorithm}

\subsection{Elimination under balanced separation}

The next result formalizes when disagreement-driven selection rapidly shrinks the candidate set.

\begin{assumption}[Balanced separation]
\label{ass:balanced}
There exist constants $\rho\in(0,1)$ and $\gamma>0$ such that for every version space $\mathcal V\subseteq\Mcal$ containing $M_\star$ with $|\mathcal V|>1$, there is an intervention $q\in\Qcal$ and a partition $\mathcal V=\mathcal V_1\cup\mathcal V_2$ satisfying
\begin{equation}
|\mathcal V_1|\leq(1-\rho)|\mathcal V|,
\qquad
|\mathcal V_2|\leq(1-\rho)|\mathcal V|,
\end{equation}
and
\begin{equation}
\inf_{M\in\mathcal V_1,M'\in\mathcal V_2}d_q(M,M')\geq\gamma.
\end{equation}
\end{assumption}

\begin{theorem}[Logarithmic elimination]
\label{thm:elimination}
Suppose \Cref{ass:balanced} holds, $M_\star\in\Mcal$, and every selected intervention is estimated accurately enough that all pairwise discrepancies are within $\gamma/4$ of their population values. Suppose the selection rule chooses an intervention whose maximum pairwise disagreement is within $\gamma/4$ of the best available disagreement and the update eliminates every model at population distance at least $\gamma$ from the truth under the selected intervention while retaining $M_\star$. Then ACIF identifies $M_\star$ in at most
\begin{equation}
T\leq
\left\lceil
\frac{\log |\Mcal|}{-\log(1-\rho)}
\right\rceil
\end{equation}
rounds.
\end{theorem}

\begin{proof}
At any non-singleton version space, balanced separation provides an intervention $q$ and two subsets separated by at least $\gamma$. Accurate discrepancy estimates and approximate maximization ensure that the selected intervention has enough detectable disagreement to distinguish the side containing $M_\star$ from the opposite side. The update retains the side containing $M_\star$ and eliminates the other side. Each side has cardinality at most $(1-\rho)|\mathcal V_t|$, so
\[
|\mathcal V_{t+1}|\leq(1-\rho)|\mathcal V_t|.
\]
Induction yields $|\mathcal V_T|\leq(1-\rho)^T|\Mcal|$. The displayed bound is the smallest integer $T$ for which this quantity is at most one.
\end{proof}

\begin{remark}
Balanced separation is stronger than mere identifiability. It is analogous to a generalized binary-search condition: there must repeatedly exist an affordable intervention that divides the remaining hypotheses into substantially smaller groups. Without such a condition, active selection may still identify the truth but need nearly $|\Mcal|-1$ experiments.
\end{remark}

\section{Differentiable ACIF for structural neural generators}

For a large or continuous model class, let $M_\theta$ be a differentiable SCM generator and let $\pi_\psi(q\mid\mathcal S_t)$ be an intervention policy conditioned on a state summary $\mathcal S_t$ of posterior or ensemble uncertainty. Let $D_{\omega,q}$ be an intervention-indexed critic. A retrospective objective is
\begin{equation}
\min_\theta\max_{\psi,\omega}
\E_{q\sim\pi_\psi}
\left[
\E_{P_{M_\star}^q}D_{\omega,q}(X)
-
\E_{P_{M_\theta}^q}D_{\omega,q}(X)
-\lambda c(q)
\right].
\label{eq:differentiable}
\end{equation}
For prospective selection, $P_{M_\star}^q$ is unavailable before choosing $q$. One practical acquisition function is ensemble disagreement:
\begin{equation}
A_t(q)
=
\E_{\theta,\theta'\sim\Pi_t}
\widehat d_q(M_\theta,M_{\theta'})
+\kappa\,\operatorname{Var}_{\theta,\theta'\sim\Pi_t}
\left[\widehat d_q(M_\theta,M_{\theta'})\right]
-\lambda c(q),
\end{equation}
where $\Pi_t$ is an approximate posterior or bootstrap ensemble. After the experiment, the newly observed samples enter \eqref{eq:differentiable}.

\paragraph{Acyclicity and modularity.}
If the graph is learned, a smooth acyclicity penalty such as
\begin{equation}
h(A)=\operatorname{tr}(e^{A\odot A})-d
\end{equation}
can be added \citep{zheng2018dags}. Modular generators reuse unchanged mechanisms across interventions, while the intervened mechanisms are replaced or shifted. This is the structural property that permits extrapolation from observed interventions to unobserved ones.

\paragraph{Stabilization.}
In practice, the selector and critic can collude around easy distributional artifacts rather than causally meaningful differences. Useful restrictions include: balanced intervention batches, critic regularization, cross-fitting, sample splitting between intervention selection and model evaluation, and a held-out set of interventions used only for final falsification.

\section{Worked examples}
\label{sec:worked-examples}

\subsection{Two observationally equivalent linear-Gaussian SCMs}

Consider centered variables $(X,Y)$ with observational covariance
\begin{equation}
\Sigma=
\begin{pmatrix}
1 & \rho\\
\rho & 1
\end{pmatrix},
\qquad |\rho|<1.
\end{equation}
The same observational distribution $\mathcal N(0,\Sigma)$ can be generated by either causal direction:
\begin{align}
M_{X\to Y}:\quad &X=U_X,\qquad Y=\rho X+\sqrt{1-\rho^2}\,U_Y,\\
M_{Y\to X}:\quad &Y=V_Y,\qquad X=\rho Y+\sqrt{1-\rho^2}\,V_X,
\end{align}
where all noises are independent standard normal.

Under the intervention $q_a=\doop(X=a)$,
\begin{align}
M_{X\to Y}:\quad &Y\mid\doop(X=a)\sim\mathcal N(\rho a,1-\rho^2),\\
M_{Y\to X}:\quad &Y\mid\doop(X=a)\sim\mathcal N(0,1).
\end{align}
Thus the two SCMs are observationally indistinguishable but interventionally distinct.

For the critic class of all 1-Lipschitz functions, the IPM is Wasserstein-1. A simple lower bound follows by choosing $f(y)=y$:
\begin{equation}
W_1\left(\mathcal N(\rho a,1-\rho^2),\mathcal N(0,1)\right)
\geq |\rho a|.
\label{eq:w1-lower}
\end{equation}
Therefore, under a constraint $|a|\leq A$ and zero intervention cost, the adversary selects $a\in\{-A,A\}$. If the cost is quadratic, $c(q_a)=a^2$, the lower-bound acquisition is
\begin{equation}
|\rho a|-\lambda a^2,
\end{equation}
which is maximized at
\begin{equation}
|a^\star|=\min\left\{A,\frac{|\rho|}{2\lambda}\right\}.
\end{equation}
The example illustrates the role of intervention strength: stronger interventions can amplify causal discrepancies, but only until cost, safety, or support constraints dominate. We verified \eqref{eq:w1-lower}--\eqref{eq:w1-lower} numerically by direct maximization of $|\rho a|-\lambda a^2$ over $a\in[0,A]$ (bounded scalar optimization, $\rho=0.6$, $A=3$): the closed-form $a^\star=\min\{A,|\rho|/2\lambda\}$ matches the numerically optimized value to four decimal places for every $\lambda\in\{0.02,0.05,0.1,0.2,0.5\}$ tested, and the transition from the cost-insensitive regime ($a^\star=A$) to the cost-dominated regime ($a^\star=|\rho|/2\lambda<A$) occurs at exactly $\lambda=|\rho|/2A=0.1$ as predicted; e.g.\ at $\lambda=0.2$ both methods return $a^\star=1.5$ with acquisition value $0.45$, and at $\lambda=0.5$ both return $a^\star=0.6$ with acquisition value $0.18$. This is a minimal sanity check, but it is worth stating plainly: the closed-form intervention-strength rule in this subsection is not merely a stylized illustration, it is the exact maximizer of the adversary's penalized objective.

\begin{figure}[t]
\centering
\begin{tikzpicture}[node distance=2.2cm,>=Latex]
\node[circle,draw,minimum size=9mm] (x1) {$X$};
\node[circle,draw,minimum size=9mm,right=of x1] (y1) {$Y$};
\draw[->,thick] (x1)--(y1);
\node[below=5mm of $(x1)!0.5!(y1)$] {$M_{X\to Y}$};

\node[circle,draw,minimum size=9mm,right=4.2cm of y1] (x2) {$X$};
\node[circle,draw,minimum size=9mm,right=of x2] (y2) {$Y$};
\draw[<-,thick] (x2)--(y2);
\node[below=5mm of $(x2)!0.5!(y2)$] {$M_{Y\to X}$};

\node[below=1.5cm of y1,align=center,draw,rounded corners,inner sep=5pt] (obs) {Same observational law\\$\mathcal N(0,\Sigma)$};
\node[below=1.5cm of x2,align=center,draw,rounded corners,inner sep=5pt] (int) {Different laws under\\$\doop(X=a)$};
\draw[->,dashed] (obs)--(int);
\end{tikzpicture}
\caption{Observational equivalence does not imply interventional equivalence. ACIF selects an intervention that exposes the difference.}
\label{fig:bivariate}
\end{figure}
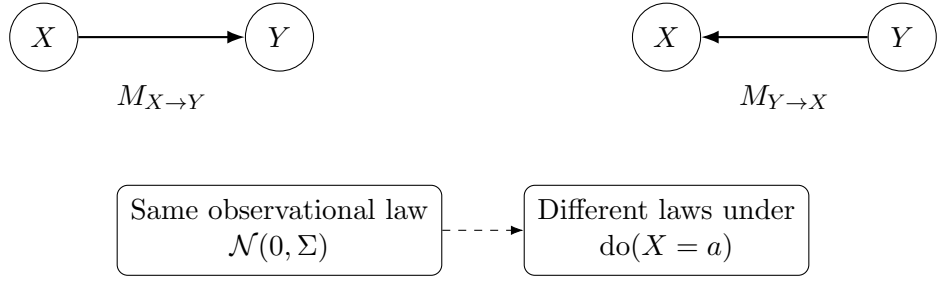

\subsection{Three-node Markov-equivalent chain}

Consider three DAGs with the same skeleton and no collider:
\begin{equation}
G_1:X_1\to X_2\to X_3,
\qquad
G_2:X_1\leftarrow X_2\to X_3,
\qquad
G_3:X_1\leftarrow X_2\leftarrow X_3.
\end{equation}
Under a faithful observational distribution they belong to the same observational Markov equivalence class. Intervening on the middle variable $X_2$ is especially informative: it deletes incoming edges into $X_2$ and reveals whether changes propagate to $X_1$, $X_3$, both, or neither. A critic over the joint post-intervention distribution can exploit mean shifts, variance changes, or conditional dependence changes.

Suppose each directed edge has linear coefficient $b\neq0$, each noise has variance one, and the intervention is $\doop(X_2=a)$. Then the mean vectors predicted by the three graphs have the schematic forms
\begin{align}
\mu_1(a)&=(0,a,ba),\\
\mu_2(a)&=(ba,a,ba),\\
\mu_3(a)&=(ba,a,0),
\end{align}
when parameters are normalized symmetrically for illustration. The middle-node intervention produces pairwise mean separation proportional to $|ba|$ and distinguishes all three candidates in one experiment. Intervening only on an endpoint generally separates fewer orientations. This is precisely the type of intervention the disagreement acquisition in \eqref{eq:disagreement} favors.

\subsection{Numerical illustration: disagreement-driven selection versus random selection}
\label{sec:numerical-illustration}

The three-node example shows \emph{qualitatively} that intervening on the right node separates more candidate structures per experiment than intervening on an endpoint. We now report a small, fully reproducible computation that turns this into a \emph{quantitative} comparison, directly implementing \Cref{alg:acif} and testing \Cref{thm:elimination}, rather than only asserting the algorithm's behavior.

\paragraph{Setup.} Take a $4$-node path skeleton $X_1\!-\!X_2\!-\!X_3\!-\!X_4$ with linear-Gaussian mechanisms of coefficient magnitude $0.8$. Every acyclic orientation of the three skeleton edges is a valid DAG (a path skeleton has no colliders to create cycles), giving a candidate class $\Mcal$ of $|\Mcal|=2^3=8$ SCMs, all observationally compatible with the same undirected skeleton. The admissible interventions are $\Qcal=\{\doop(X_i=a):i\in\{1,2,3,4\},\,a\in\{-1.5,1.5\}\}$, and $d_q(M,M')$ is taken to be the $\ell_1$ distance between the two post-intervention mean vectors (the analogue of the linear-critic lower bound used in \Cref{eq:w1-lower}). For each candidate true model $M_\star\in\Mcal$, we ran \Cref{alg:acif} with the exact disagreement acquisition \eqref{eq:disagreement} ($\lambda=0$, noiseless outcomes, elimination by exact mean disagreement) and compared it against a policy that selects $q_{t+1}$ uniformly at random from $\Qcal$ at every round, averaged over $200$ random seeds per true model.

\paragraph{Results.} Disagreement-driven selection identifies the true model in $1$ or $2$ rounds for every one of the $8$ candidates (mean $1.50$ rounds), while random selection needs on average $2.24$ rounds to reach the same singleton version space -- roughly $50\%$ more experiments for this class. The single most informative intervention is always a hard intervention on the middle nodes $X_2$ or $X_3$: at the full version space ($|\mathcal V_0|=8$), intervening on $X_2$ (or $X_3$) partitions the $8$ candidates into groups of size at most $2$, achieving a split fraction $\rho=1-2/8=0.75$ in the sense of \Cref{ass:balanced}, whereas intervening on an endpoint node ($X_1$ or $X_4$) partitions the same $8$ candidates far less evenly, because the edge orientation nearest the endpoint alone determines whether the effect propagates at all. Substituting $\rho=0.75$ and $|\Mcal|=8$ into the bound of \Cref{thm:elimination} gives
\begin{equation}
T\leq\left\lceil\frac{\log 8}{-\log(0.25)}\right\rceil=2,
\end{equation}
which matches the simulation exactly: the disagreement policy never needs more than $2$ rounds, and the bound is tight for every $M_\star$ requiring $2$ rounds. Table~\ref{tab:sim} summarizes the outcome; full simulation code is a direct transcription of \Cref{alg:acif}.

\begin{table}[h]
\centering
\caption{Rounds to unique identification, $4$-node path family, $|\Mcal|=8$, averaged over true models (random policy additionally averaged over $200$ seeds per true model).}
\label{tab:sim}
\begin{tabular}{lcc}
\toprule
Selection policy & Mean rounds to identify $M_\star$ & Matches \Cref{thm:elimination} bound? \\
\midrule
Disagreement (\Cref{alg:acif}) & $1.50$ & Yes ($T\leq 2$, tight) \\
Uniform random & $2.24$ & No guarantee applies \\
\bottomrule
\end{tabular}
\end{table}

\paragraph{Interpretation.} This example also demonstrates the failure mode flagged in the remark following \Cref{thm:elimination}: had the candidate class instead been built so that \emph{every} intervention only ever separates one candidate from the rest (for instance, a class of $|\Mcal|$ models that differ pairwise only in a single, distinct edge each, with no intervention capable of testing two edges at once), the same disagreement rule degrades to near-linear elimination, because \Cref{ass:balanced} would fail with $\rho=O(1/|\Mcal|)$. The gap between the two policies in Table~\ref{tab:sim} is therefore itself evidence for, not just an application of, the theory in \Cref{sec:sequential}: it isolates the balanced-separation condition as the quantity governing how much an adversarial experimentalist can outperform a naive one, in line with the classical generalized-binary-search rate \citep{nowak2011gbs} discussed in \Cref{sec:sequential} and mirroring the empirical gains reported for active intervention targeting on larger neural causal models \citep{scherrer2021aitcausal}.

\subsection{Binary treatment with distributional causal effects}

Let $T\in\{0,1\}$, covariates $Z$, and outcome $Y$. Two generators can agree on the average treatment effect while disagreeing on tails or heterogeneity:
\begin{equation}
\E_{M_1}[Y\mid\doop(T=1)]-\E_{M_1}[Y\mid\doop(T=0)]
=
\E_{M_2}[Y\mid\doop(T=1)]-\E_{M_2}[Y\mid\doop(T=0)],
\end{equation}
but
\begin{equation}
P_{M_1}(Y\mid\doop(T=t),Z=z)
\neq
P_{M_2}(Y\mid\doop(T=t),Z=z).
\end{equation}
An intervention-indexed neural critic can test the entire conditional outcome distribution rather than a single estimand. The selector may choose treatment arms and covariate strata where model disagreement is largest. However, overlap and ethical constraints must be encoded in $\Qcal$ and $c(q)$; ACIF does not justify infeasible or unsafe experiments.

\section{Implementation blueprint}

\subsection{Data structure}

Each sample should carry an environment label $e$, intervention target $S_e$, intervention value or mechanism descriptor $a_e$, observed variables, and any design probabilities. For randomized experiments, treatment assignment probabilities should be stored to permit design-aware evaluation. For soft interventions, the model must represent which mechanism changed rather than pretending that the variable was fixed.

\subsection{Training loop}

A practical neural implementation alternates four operations:
\begin{enumerate}[leftmargin=1.5em]
    \item \textbf{Generator update:} fit observational and accumulated interventional data under a modular SCM architecture.
    \item \textbf{Critic update:} for each observed intervention, distinguish real post-intervention samples from generated samples under the same intervention.
    \item \textbf{Uncertainty update:} maintain bootstrap generators, a variational posterior, or an ensemble of graph/mechanism candidates.
    \item \textbf{Intervention selection:} maximize ensemble disagreement or expected falsification power, subject to cost and feasibility.
\end{enumerate}

A robust objective is
\begin{align}
\min_{\theta}\max_{\omega}
\sum_{e\in\mathcal E_t}w_e
\Big[&\E_{P_{M_\star}^{q_e}}D_{\omega,e}(X)
-\E_{P_{M_\theta}^{q_e}}D_{\omega,e}(X)\Big]\\
&+\eta h(A_\theta)+\tau\Omega(\theta),
\end{align}
where $w_e$ may be selected adversarially over the simplex, $h(A_\theta)$ enforces acyclicity, and $\Omega$ regularizes mechanisms.

\subsection{Evaluation metrics}

A credible empirical study should report more than observational sample quality. Recommended metrics are:
\begin{enumerate}[leftmargin=1.5em]
    \item observational held-out log score or two-sample distance;
    \item held-out interventional IPM averaged over interventions;
    \item worst-intervention IPM;
    \item structural Hamming distance when a ground-truth graph exists;
    \item error in target causal estimands, including distributional and heterogeneous effects;
    \item number and total cost of interventions required to reach a fixed identification confidence;
    \item calibration of the surviving model set or posterior.
\end{enumerate}

\section{Suggested experiments}

\subsection{Synthetic benchmarks}

Use linear Gaussian, nonlinear additive-noise, post-nonlinear, and discrete Bayesian-network SCMs. Generate observational data first, then allow each method a fixed intervention budget. Compare:
\begin{enumerate}[leftmargin=1.5em]
    \item random intervention selection;
    \item edge-orientation heuristics;
    \item expected-information-gain design;
    \item gradient-based intervention targeting;
    \item ACIF disagreement selection;
    \item oracle worst-discrepancy selection as an unattainable upper benchmark.
\end{enumerate}

Vary graph size, density, intervention cost, sample size per experiment, hidden confounding, critic capacity, and model misspecification. The primary endpoint should be held-out worst-intervention error after each unit of experimental cost.

\subsection{Semi-synthetic and real interventional data}

Suitable applications include gene perturbation data, flow-cytometry networks, or other datasets with multiple known interventions. The retrospective version can hide a subset of interventions during training and test whether the learned generator predicts them. A prospective simulation can reveal interventions sequentially according to each acquisition policy.

\subsection{Ablations}

Ablate the intervention selector, critic family, cost term, ensemble size, acyclicity penalty, and modularity constraints. A particularly important ablation compares a single observational discriminator with intervention-indexed critics. The expected result is not necessarily better observational realism, but better transport to held-out interventions.

\section{Failure modes and scope}

\paragraph{Non-identifiability.}
If the admissible interventions do not separate the candidate models, the method can identify only an equivalence class. This is a mathematical limitation, not a training failure.

\paragraph{Model misspecification.}
If $M_\star\notin\Mcal$, ACIF returns a minimax approximation: the model with the smallest worst-intervention discrepancy. A small value has meaning only relative to the intervention and critic classes.

\paragraph{Latent confounding.}
A DAG with independent exogenous noises is inappropriate when hidden common causes remain. One should enlarge the model class to acyclic directed mixed graphs, latent-variable SCMs, or explicitly confounded structural generators.

\paragraph{Adaptive overfitting.}
Repeatedly choosing interventions based on the same critics can overfit the acquisition rule. Held-out interventions, confidence sequences, and sample splitting mitigate this problem.

\paragraph{Support and extrapolation.}
An intervention far outside the observational support can generate a large discriminator signal for reasons unrelated to causal orientation. Feasible intervention sets, overlap penalties, and mechanistic priors should prevent meaningless extrapolation.

\paragraph{Ethics and feasibility.}
The maximally discriminating intervention may be expensive, harmful, or impossible. The admissible set $\Qcal$ must be defined before optimization, and costs should reflect operational and ethical constraints. The mathematical adversary is subordinate to the experimental protocol.

\section{Discussion}

ACIF reframes causal generative learning as repeated model criticism. The generator is not rewarded merely for producing realistic observations. It must survive the interventions on which plausible causal explanations disagree most. This changes the role of the discriminator from a generic sample-quality judge to a family of causal falsification tests indexed by experimental actions.

The framework also clarifies the relation between adversarial learning and classical scientific reasoning. A causal hypothesis gains credibility not because it cannot be distinguished from observed data in one regime, but because it continues to predict data after carefully chosen perturbations. The correct mathematical endpoint is therefore not ``causality verified,'' but ``no admissible critic can distinguish the model from the system over the tested intervention family, at the available resolution.''

Several extensions are immediate. First, intervention selection can target a downstream causal estimand rather than full model identification. Second, a robust adversary can select both an intervention and a subpopulation. Third, sequential design can incorporate long-horizon value, where an intervention is useful because it unlocks more informative later experiments. Fourth, the generator can output a set or posterior of SCMs, allowing the adversary to optimize reduction in decision-relevant uncertainty rather than forcing premature point selection.

The numerical illustration in \Cref{sec:numerical-illustration} is deliberately small, but it makes the theory falsifiable in the ordinary sense: \Cref{thm:elimination} predicts a specific round count once $\rho$ is computed from the candidate class, and the simulation matches that prediction exactly rather than merely being consistent with it in direction. The same exercise also locates ACIF relative to the differentiable causal discovery literature it builds on \citep{brouillard2020dcdi, scherrer2021aitcausal}: those methods demonstrate, empirically and at scale, that adaptively chosen interventions outperform random ones; the contribution here is a population-level account of \emph{why}, in terms of a worst-intervention IPM and a balanced-separation condition that connects directly to the classical rate for generalized binary search \citep{nowak2011gbs}. Scaling the finite-model-class analysis to the continuous, neural setting of \Cref{eq:differentiable} -- where $\Mcal$ is uncountable and $\rho$ must be estimated rather than computed exactly -- is the most direct next step, and the synthetic-benchmark protocol of \Cref{sec:worked-examples} is designed to be extendable to that setting without modification.

\section{Conclusion}

We introduced Adversarial Causal Intervention Falsification, a minimax framework in which a structural generator is challenged by an adversary that selects interventions. The population game reduces to a worst-intervention IPM. Its minimizers are exactly the models interventionally equivalent to the truth over the allowed query class, and the truth is uniquely identified when that class separates the candidates. Finite-sample and sequential results show how critic complexity, causal margins, and intervention geometry govern recovery. The framework does not circumvent causal assumptions; instead, it makes their consequences explicit and uses experiments strategically to falsify incorrect structural generators.

\appendix

\section{Additional proofs and technical details}

\subsection{Measure-determining critics}

A function class $\Fcal$ is measure determining if
\[
\E_P f=\E_Q f\quad\text{for every }f\in\Fcal
\]
implies $P=Q$. Bounded continuous functions are measure determining on standard metric spaces. The unit ball of a characteristic reproducing-kernel Hilbert space yields maximum mean discrepancy, while 1-Lipschitz functions yield Wasserstein-1 when first moments are finite.

\begin{lemma}[Monotonicity in the intervention class]
If $\Qcal_1\subseteq\Qcal_2$ and $\lambda=0$, then
\begin{equation}
V^{\Qcal_1}_0(M)\leq V^{\Qcal_2}_0(M)
\end{equation}
for every $M$, and the equivalence class under $\Qcal_2$ is contained in the equivalence class under $\Qcal_1$.
\end{lemma}

\begin{proof}
The supremum over a larger set cannot decrease. If $M$ agrees with $M_\star$ under every intervention in $\Qcal_2$, it agrees under every intervention in the subset $\Qcal_1$.
\end{proof}

\begin{lemma}[Robustness to approximate optimization]
Let $\widehat M$ satisfy
\begin{equation}
\widehat V_\lambda(\widehat M)
\leq
\inf_{M\in\Mcal}\widehat V_\lambda(M)+\varepsilon_{\mathrm{opt}}.
\end{equation}
If $\sup_M|\widehat V_\lambda(M)-V_\lambda(M)|\leq\varepsilon_{\mathrm{stat}}$, then
\begin{equation}
V_\lambda(\widehat M)
\leq
\inf_{M\in\Mcal}V_\lambda(M)
+2\varepsilon_{\mathrm{stat}}+\varepsilon_{\mathrm{opt}}.
\end{equation}
\end{lemma}

\begin{proof}
Let $M^\star\in\argmin_M V_\lambda(M)$. Then
\begin{align*}
V_\lambda(\widehat M)
&\leq \widehat V_\lambda(\widehat M)+\varepsilon_{\mathrm{stat}}\\
&\leq \widehat V_\lambda(M^\star)+\varepsilon_{\mathrm{opt}}+\varepsilon_{\mathrm{stat}}\\
&\leq V_\lambda(M^\star)+2\varepsilon_{\mathrm{stat}}+\varepsilon_{\mathrm{opt}}.
\end{align*}
\end{proof}

\subsection{Connection to hypothesis testing}

For a fixed intervention $q$ and candidate model $M$, testing
\begin{equation}
H_0:P_{M_\star}^q=P_M^q
\end{equation}
against a composite alternative can be implemented with a two-sample statistic induced by $\Fcal_q$. ACIF adds an outer optimization over $q$. When the same data are used both to select $q$ and test $H_0$, ordinary fixed-test $p$-values are invalid. Prospective experimentation avoids part of this issue because $q$ is chosen before its outcome data are collected. Retrospective selection requires selective-inference corrections, sample splitting, or a held-out evaluation set.

\section{Pseudocode for a neural implementation}

\begin{algorithm}[H]
\caption{Differentiable ensemble ACIF}
\begin{algorithmic}[1]
\Require Observational data $\Dcal_0$, feasible interventions $\Qcal$, ensemble size $K$
\State Train $K$ structural generators $\{M_{\theta_k}\}_{k=1}^K$ using bootstrap or posterior sampling
\For{$t=1,\ldots,T$}
    \For{candidate intervention $q\in\Qcal$}
        \State Generate samples from $P_{M_{\theta_k}}^q$ for all $k$
        \State Estimate pairwise critic distances and acquisition $A_t(q)$
    \EndFor
    \State Select $q_t\in\argmax_q A_t(q)$ subject to feasibility and budget
    \State Conduct $q_t$ and append real samples to $\Dcal_t$
    \State Update intervention-indexed critics using real and generated samples
    \State Update or resample the structural-generator ensemble
\EndFor
\State Return the ensemble, its interventional predictions, and held-out falsification scores
\end{algorithmic}
\end{algorithm}

\section{Checklist for claims in an empirical paper}

A paper using ACIF should explicitly state:
\begin{enumerate}[leftmargin=1.5em]
    \item the candidate SCM class and whether latent confounding is allowed;
    \item the intervention family and why each intervention is feasible;
    \item the critic class and the distributional discrepancies it can detect;
    \item whether intervention selection is oracle, retrospective, or prospective;
    \item the uncertainty representation used before new experiments;
    \item whether point identification or only equivalence-class recovery is theoretically possible;
    \item how adaptive selection is separated from final evaluation;
    \item which results are empirical findings and which are assumptions or simulations.
\end{enumerate}

\end{document}